\documentclass[letterpaper]{article} 
\usepackage{aaai24}  
\usepackage{times}  
\usepackage{helvet}  
\usepackage{courier}  
\usepackage[hyphens]{url}  
\usepackage{graphicx} 
\usepackage{natbib}  
\usepackage{caption} 

\usepackage{mathrsfs}
\usepackage{amsmath,amsthm}
\usepackage{amssymb}
\usepackage{courier}
\usepackage{verbatim}
\usepackage{xcolor}
\usepackage{amsmath, bm}
\usepackage{tikz}
\usepackage{appendix}
\usepackage{subcaption}
\usepackage{boondox-cal}
\usepackage{multirow}

\usepackage{enumitem}
\usepackage{array} 
\usepackage{subcaption}
\usepackage{bibunits}
\usepackage{bibentry}

\newtheorem{proposition}{Proposition}
\makeatletter
\newif\if@restonecol
\makeatother

\usepackage{algorithmicx}
\usepackage[ruled,linesnumbered]{algorithm2e}
\usepackage{mathrsfs}
\usepackage{newfloat}
\usepackage{listings}
\usepackage{float}
\DeclareCaptionStyle{ruled}{labelfont=normalfont,labelsep=colon,strut=off} 
\floatstyle{ruled}
\newfloat{listing}{tb}{lst}{}
\floatname{listing}{Listing}
\title{Taming Gradient Variance in Federated Learning with Networked Control Variates}
\author{
    Xingyan Chen \textsuperscript{\rm 1},
    Yaling Liu\textsuperscript{\rm 1},
    Huaming Du\textsuperscript{\rm 1},
    Mu Wang\textsuperscript{\rm 2}$^{\dagger}$,
    Yu Zhao\textsuperscript{\rm 1}
}
\affiliations{
    \textsuperscript{\rm 1} Institute of Digital Economy and Interdisciplinary Science Innovation\\

    School of Computer and Artificial Intelligence\\
    Southwestern University of Finance and Economics\\
    Chengdu 611130, P. R. China.\\
     \textsuperscript{\rm 2} the State Key Laboratory of Networking and Switching Technology\\
     Beijing University of Posts and Telecommunications\\

     Beijing 100876, P. R. China.\\
     $^{\dagger}$Corresponding authors.
}

\usepackage{bibentry}

\begin{document}
\begin{bibunit}
\maketitle

\begin{abstract}
Federated learning, a decentralized approach to machine learning, faces significant challenges such as extensive communication overheads, slow convergence, and unstable improvements.
These challenges primarily stem from the gradient variance due to heterogeneous client data distributions.
To address this, we introduce a novel Networked Control Variates (FedNCV) framework for Federated Learning.
We adopt the REINFORCE Leave-One-Out (RLOO) as a fundamental control variate unit in the FedNCV framework, implemented at both client and server levels.
At the client level, the RLOO control variate is employed to optimize local gradient updates, mitigating the variance introduced by data samples.
Once relayed to the server, the RLOO-based estimator further provides an unbiased and low-variance aggregated gradient, leading to robust global updates.
This dual-side application is formalized as a linear combination of composite control variates. We provide a mathematical expression capturing this integration of double control variates within FedNCV and present three theoretical results with corresponding proofs.
This unique dual structure equips FedNCV to address data heterogeneity and scalability issues, thus potentially paving the way for large-scale applications.
Moreover, we tested FedNCV on six diverse datasets under a Dirichlet distribution with \( \alpha = 0.1 \), and benchmarked its performance against six SOTA methods, demonstrating its superiority.
\end{abstract}

\section{Introduction}
Federated Learning (FL) stands as a promising distributed learning framework in large-scale scenarios. Contrasting traditional methods that train models on vast datasets stored in a central server, FL enables the training of a unified model without the need for network transmission of client data, thereby upholding client data privacy. 
In particular, FL solutions, such as FedAvg \citep{konevcny2016federated} and its variants \citep{kairouz2021advances,mcmahan2017communication}, utilize strategies like parameter averaging or cumulative gradient updates, which reduce network communication costs compared to traditional distributed learning approaches. However, the promising performance of FL is tempered by several intrinsic challenges.
The issues stem from the heterogeneity of client data, characterized by non-identically and independently distributed (non-IID) patterns \citep{collins2021exploiting,arivazhagan2019federated,tan2023pfedsim}, non-stationarity of data generation \citep{karimireddy2020scaffold}, and heavier-tailed noise introduced by lower-quality data \citep{yang2022taming}. 
These challenges often result in considerable variance in the cumulative gradient during both local training and global model updating, thereby undermining the efficiency of FL. As a consequence, these issues can lead to slower convergence and extended model training periods, resulting  in a significant waste of computational and bandwidth resources \citep{mcmahan2017communication,lim2020federated}.

To tackle the above challenges, a number of studies have conducted in-depth investigations into effective FL algorithms \citep{li2020federated,karimireddy2020scaffold,wang2020tackling,collins2021exploiting,arivazhagan2019federated,yang2022taming,tan2023pfedsim}.
FedProx \citep{li2020federated} incorporates a regularization term to enhance model generalization, while FedNova \citep{wang2020tackling} introduces scaling and normalization updates to dynamically alter the FL process. 
SCAFFOLD \citep{karimireddy2020scaffold} identifies a ``client-side bias'' in FedAvg due to non-IID features, leading to unstable improvements and slow convergence. By applying control variates to model parameters during local updates, this approach effectively reduces gradient variance, accelerating algorithm convergence.
Moreover, several other studies have presented unique solutions: 
FedRep \citep{collins2021exploiting} introduces a novel federated learning framework that leverages shared data representation across clients, effectively addressing statistical heterogeneity.
FedPer \citep{arivazhagan2019federated} employs a ``base + personalization layer'' strategy for model training. This structure counters statistical heterogeneity, ensuring personalization.
A FAT-Clipping framework \citep{yang2022taming} is introduced to tackle fat-tailed noise challenges in data samples.
pFedSim \citep{tan2023pfedsim} advances personalized federated learning by decoupling a neural network into two parts: a personalized feature extractor, derived from clustering similar client models, and a local classifier. 


However, while several methods address convergence challenges from various angles, the central issue may not have been fully considered. The slow convergence observed in FL architectures, especially under the backpropagation framework, is potentially due to the bias and instability of the average-based gradient estimator.
The variance introduced by an inaccurate gradient estimator can hinder the algorithm's convergence. 
This challenge persists irrespective of the employment of regularization \citep{li2020federated}, normalization \citep{wang2020tackling}, representation \citep{collins2021exploiting}, architecture refinement \citep{arivazhagan2019federated}, clustering techniques\citep{tan2023pfedsim}, or noise mitigation \citep{yang2022taming}.
Moreover, advanced gradient estimators \citep{titsias2022double,grathwohl2017backpropagation} have not been widely adopted in the FL paradigm. This limitation arises because FL operates on a distributed structure, which presents challenges for integrating these techniques.
For instance, SCAFFOLD integrates basic control variates based gradient estimator only during local training. Lacking a systematic design, it's vulnerable to the influence of the number of nodes and model parameters, thus affecting its efficiency.

In this context, we propose a novel compositional structure for control variates, called Networked Control Variates (FedNCV), specifically designed for the FL paradigm.
FedNCV, with a design perspective focused on the variance of stochastic gradients, aims to tackle the challenges of data heterogeneity and large-scale training. 
A critical element in the FedNCV framework is the REINFORCE Leave-One-Out (RLOO) control variate.
By deploying control variates on both the server and client sides, FedNCV effectively manages gradient variance during local training and also ensures stability in gradient aggregation on the server side in large-scale client environments.
This dual-side application can be formalized as a linear combination of multiple control variates, with its design explained by the double control variates theory \citep{titsias2022double}. 
We provide a mathematical expression that captures this integration of composite control variates in FedNCV and present three theoretical results with corresponding proofs.
This compositional structure enables FedNCV to address issues of data heterogeneity and scalability, potentially opening new avenues for large-scale FL applications.
Furthermore, under the Dirichlet distribution with \( \alpha = 0.1 \), 
we assessed FedNCV across four real-world datasets and juxtaposed its efficacy with six leading-edge solutions, revealing its enhanced performance.

\section{System Model and Preliminaries}
This section presents our network model and provides fundamental insights into the two main techniques implemented: federated learning via gradient aggregation and gradient estimation employing control variates.

\subsection{Network Model}
The heterogeneous network architecture is modeled as an undirected graph, $\mathcal{G}(\mathcal{V},\mathcal{E})$, composed of $|\mathcal{V}|$ nodes and $|\mathcal{E}|$ edges. 
The nodes encompass a central parameter server (${v_s}\in\mathcal{V}$) and edge workers ($\mathcal{V}_e \subset \mathcal{V}$), with $n_e=|\mathcal{V}_e|$ indicating the count of edge workers.
We set the number of the edge workers as $n_e=|\mathcal{V}_e|$.
To improve efficiency, a slotted time model, represented as $\mathcal{T} = \{1,2,...\}$, is employed, granting computational capabilities to each node $u \in \mathcal{V}$.
Each edge worker is associated with an online local dataset, denoted as $\mathcal{D}_u$ for worker node $u \in \mathcal{V}_e$ since the data samples are generated on the fly in an online fashion. Every dataset $\mathcal{D}_u$ is independently and identically distributed, comprising $n_u$ data samples, i.e., $\mathcal{D}_u = \{(x_i, y_i)\}_{i=1}^{n_u}$, where $x_i$ signifies the feature vector and $y_i$ denotes the corresponding label. The total quantity of samples across all workers equals $n = \sum_{u \in \mathcal{V}_e} n_u$.
Each edge worker $u \in \mathcal{V}_e$ is characterized by a local objective function $\mathcal{L}_u(\cdot)$, quantifying the discrepancy between the model's predictions and the actual labels of the local dataset $\mathcal{D}_u$. A common choice for $\mathcal{L}_u$ is the empirical risk function as:
\begin{equation}
\mathcal{L}_u\left(\theta; \mathcal{D}_u\right) = \sum_{(x_{i},y_{i}) \in \mathcal{D}_u} \mathcal{l}\left(\theta; x_{i}, y_{i}\right),
\end{equation}
where $\mathcal{l}(\theta; x_{i}, y_{i})$ is the loss function, capturing the discrepancy between the model's prediction using parameters $\theta$ and the actual label $y_{i}$ for a given feature vector $x_{i}$.

\subsection{Federated Learning via Gradient Aggregation}
Federated learning strategies, such as FedAvg \citep{mcmahan2017communication}, are fundamentally grounded in a distributed learning framework that consists of a central parameter server and numerous computational workers, denoted as $\mathcal{V}_e$.
At the beginning of each round $t$, the central server disseminates the model parameters, denoted as $\theta_t$, to the worker nodes. Each worker $u \in \mathcal{V}_e$ is then tasked with calculating multi-step gradients $g_{u}$ as follows:
\begin{equation}\small
g_{u} = \nabla_{\theta_t}\mathcal{L}_u\left(\theta_t; \mathcal{D}_u \right) = \nabla_{\theta_t} \sum_{(x_i,y_i)\in \mathcal{D}_u} \mathcal{l}_u (\theta_t; x_i, y_i). 
\label{local_gradient}
\end{equation}
After accumulating all gradients, the parameter server updates the model parameters using a weighted average approach based on the number of local data samples $n_u$:
\begin{equation}
\theta_{t+1} \leftarrow \theta_t - \gamma \sum_{u\in \mathcal{V}_e} \frac{n_u}{n} g^u_t,
\end{equation}
where $\gamma$ denotes the learning rate, $n_u$ refers to the number of samples for worker $k$, and $n = \sum_{u\in \mathcal{V}_e} n_u$ represents the total sample count across all workers.
We denote $p_u = \frac{n_u}{n}$ as the $u$-th worker's sample proportion.
The training process of federated learning via gradient aggregation generally consists of the following four steps:
\begin{itemize}
\item \textbf{Global parameters broadcasting}: The central server distributes the global model parameters $\theta_t$ to all edge workers for further local training.
\item \textbf{Local model updating}: After receiving the aggregated parameters from the server, all workers update their local models and evaluate their performance using local data.
\item \textbf{Local model training}: Each worker calculates the multi-step gradients based on its local data $\mathcal{D}_u$ and returns the computed gradients to the central parameter server.
\item \textbf{Global model updating}: The central server aggregates the received gradients and updates the global model parameters without disclosing any local information.
\end{itemize}

After an adequate number of iterative local training cycles and exchanges of updates between the central server and its associated workers, the federated learning solution converges towards the globally optimal learning model.
The optimization problem can be expressed as:

\begin{equation}
\theta^* = \arg \min_{\theta} \sum_{u\in \mathcal{V}_e} p_u \mathcal{L}_u \left( \theta; \mathcal{D}_u \right)
\end{equation}

\subsection{Gradient Estimator with Control Variates}
In numerous optimization problems found within the domains of machine learning and reinforcement learning, the principal objective is to identify the optimal policy, denoted by \( p_{\theta}(x) \), that minimizes or maximizes the objective function \( \mathbb{E}_{p_{\theta}(x)}\left[\mathcal{l}(x)\right] \). Here, this objective function signifies the expected value of under the given distribution \( p_{\theta}(x) \).
Such optimizations are manifest in various contexts. For instance, in variational inference, \( \mathcal{l}(x) \) represents the evidence lower bound, and \( p_{\theta}(x) \) is recognized as the variational distribution \citep{blei2017variational}. Similarly, these optimizations occur in reinforcement learning, where \( \mathcal{l}(x) \) acts as the reward, and \( p_{\theta}(x) \) functions as the policy \citep{tao2001multi}.
A pivotal aspect of these problems includes the pursuit of unbiased gradient estimators that possess low variance, which is essential for stochastic gradient descent. The REINFORCE or score function gradient estimator encapsulates this requirement \citep{sutton2018reinforcement}, expressed as:
\begin{equation}
\nabla_{\theta}\mathcal{L}(x) =\mathbb{E}_{p_{\theta}}[\mathcal{l}(x)\nabla_{\theta}\log p_{\theta}(x)].
\end{equation}
Despite this method’s proficiency in providing an unbiased estimation of the gradient $g(\mathcal{l})=\mathcal{l}(x)\nabla_{\theta}\log p_{\theta}(x)$, it presents an inherent challenge due to its high variance owing to the sampled variables. Furthermore, it merely utilizes the output information from $\mathcal{l}(x)$, neglecting the dependency between $\mathcal{l}(x)$ and the variable $x$. To address these issues, a technique known as control variates has been introduced. The formulation of the gradient estimator then becomes:
\begin{equation}
z\left(\mathcal{l}\right) = g\left(\mathcal{l}\right) - \alpha \cdot\left(c\left(x\right) - \mathbb{E}\left[c\left(x\right)\right]\right).
\label{CV}
\end{equation}
Here, $z(\mathcal{l})$ represents the control variate estimator, $c(x)$ is a known random variable that is related to $g\left(\mathcal{l}\right)$,  and $\alpha$ is a scalar hyperparameter. This novel estimator retains the same expectation as its predecessor, given that $\mathbb{E}[c(x)]$ is a constant and the expression $\mathbb{E}[c(x) - \mathbb{E}[c(x)]]$ equates to zero. Consequently, if $g(\mathcal{l})$ delivers an unbiased gradient estimation, the same is valid for $z(\mathcal{l})$.
Furthermore, assuming significant correlation between $g(\mathcal{l})$ and $c(x)$, an optimal $\alpha$ can minimize the variance of $z(\mathcal{l})$ \citep{grathwohl2017backpropagation}. The variance of $z(\mathcal{l})$ is computed as follows:
\begin{equation}
\text{Var}[g(\mathcal{l})] + \alpha^{2}\text{Var}[c(x)] - 2\alpha \text{Cov}[g(\mathcal{l}), c(x)].
\label{var}
\end{equation}
Clearly, the effectiveness of control variates strongly depends on the correlation between $g(\mathcal{l})$ and $c(x)$.

\subsection{The REINFORCE Leave-One-Out Estimator}
The identification of appropriate control variates functions can be effectively accomplished through the REINFORCE Leave-One-Out (RLOO) estimator. This technique, specifically shaped for substantial variance reduction, operates successfully when the sample size $K\geq2$. The conventional representation of the RLOO estimator is:
\begin{equation}
\nabla_{\theta}\mathcal{L}(x) = \frac{1}{n} \sum_{i=1}^{n} \left( f_{\eta}(x_i) - \alpha c_{\mathcal{D}\setminus i}(x) \right) \nabla_{\theta} \log p_{\theta}(x), \label{RLOO}
\end{equation}
where $\mathcal{D}$ symbolizes the set of samples, and the leave-one-out control variate is given by $c_{\mathcal{D}\setminus i}(x) = \frac{1}{n-1}\sum_{j \neq i}f_\eta(x_j)$. 
Although a correction expectation-term $\mathbb{E}[c(x)]$ (as seen in \eqref{CV}) is typically incorporated to ensure unbiasedness, there are instances where it can be omitted as it will yield an overall zero expectation.
This is substantiated by the fact that if $E_{p_{\theta}(x)}[c_{\mathcal{D}\setminus i}(x)]$ amounts to a constant, then the expectation term of control variates 
 $E_{p_{\theta}(x)}[c_{\mathcal{D}\setminus i}(x)  \nabla_{\theta} \log p_{\theta}(x)] = 0$ \citep{titsias2022double}.

\section{Federated Learning with Networked Control Variates}
In this section, we will present a methodology incorporating networked control variates into federated learning. Specifically, we will first delineate the principle of networked control variates and subsequently describe its application in the context of federated learning.
\subsection{Networked Control Variates in Federated Learning}
The networked control variates method is fundamentally designed to aggregate gradients from multiple workers in a distributed learning fashion. Its goal is to minimize the variance of the cumulative gradient while ensuring that the gradient remains unbiased.
In the context of federated learning, we propose using \( p_\theta(x_i, y_i) \) as the policy model, which can also be set as a neural network model. 
This model often influences the distribution of the client-generated sample data, especially in many fields such as generative tasks, reinforcement learning and variational inference.
An example of this influence is in automatic text completion, where a single phrase might have multiple expressions. The model's suggestions could alter the client's final output, thus modifying the sample data distribution.
We can conclude that \( p_\theta(x_i, y_i) \) significantly affects the distribution of data samples and can be considered as a weighting factor for the data point \( (x_i, y_i) \). We define the sampling probability of \textit{i}-th samples for \textit{u}-th worker as $p_{\theta}^{u,i},(x_i,y_i)\in\mathcal{D}_u$ and the sampling probability of \textit{i}-th samples as $p_{\theta}^i = \sum_{u\in\mathcal{V}_e}p_{\theta}^{u,i}$.

In this context, $\mathcal{l}(x)$ can be regarded as the local loss function, computed for each data point within the local dataset $\mathcal{D}_u$. Hence, the primary objective of the distributed learning process setup becomes to minimize the global objective function, which is the expected value of the local loss function across the entire data distribution: $\mathbb{E}_{p_{\theta}}\left[\mathcal{l}_u(x)\right]$, where $\mathcal{l}_u(x)$ represents the loss for the $u$-th worker.


In our NCV-based FL framework, each client first applies the RLOO to reconstruct the gradient of each local sample. These reconstructed gradients are then aggregated using the networked control variates method, ensuring respect for individual data distributions and reducing the overall gradient variance. The server further performs the $m$-th round normalization to rectify the sampling probability bias.

Each worker $u$ computes the integrated gradient of all local samples $\mathcal{D}_u$. The $i$-th sample's gradient is initially computed according to the equation $g^{i}_u = \nabla_{\theta_t} \sum_{(x_i,y_i)\in \mathcal{D}_u} \mathcal{l}_u (\theta_t; x_i, y_i)$.
Once the gradients has been calculated, the worker $u$ reshapes the gradients using the RLOO-based control variates method:
\begin{equation}
g'^{i}_u = \left(g^{i}_u - \alpha c_{\mathcal{D}_u\setminus i} \right), \forall (x_i,y_i)\in\mathcal{D}_u,
\label{weight_u}
\end{equation}
where $c_{\mathcal{D}_u\setminus i} =  \frac{1}{n_u - 1} \sum_{j \neq i} g^{j}_u$ signifies the leave-one-out control variate. Since $c_{\mathcal{D}_u\setminus i}$ does not depend on the current sample $x_i$, $\mathbb{E}_{p_{\theta}(x_i)}[c_{\mathcal{D}_u\setminus i}]$ is equal to a constant. Thus, the expectation term $\mathbb{E}_{p_{\theta}(x_i)}[c_{\mathcal{D}_u\setminus i}  \nabla_{\theta} \log p_{\theta}(x_i)] = 0$, we can safely omit this expectation term from the computation.

Following this, each worker $u$ communicates the expectation gradient $g_{u} = \frac{1}{n_u}\sum_{i=1}^{n_u} g'^{i}_u$ back to the central server. Each worker's gradient is then processed once more using the RLOO method in the central server. This is expressed mathematically as following:
\begin{equation}
g'_{u} = \left(g_{u} - c_{\mathcal{V}_e \setminus u} \right), \forall u\in\mathcal{V}_e.
\label{server_cv}
\end{equation}
Here, $c_{\mathcal{V}_e \setminus u} = \sum_{v \neq u} \frac{n_v}{n-n_u} \cdot g_{v}$ signifies the control variates within each worker and $\frac{n_v}{n-n_u}$ is the weight factor related to the number of data samples for each worker. Similar to equation \eqref{weight_u}, the expectation term can be disregarded since it totals zero.
Next, the server estimates the global gradient $g$ by aggregating the workers' reshaped gradients $g'_{u}$ with respect to their assigned weights $p_u$. The update for the global model is given as follows:
\begin{equation}
\theta_{t+1} \leftarrow \theta_t - \gamma \cdot g,
\label{NCV_update}
\end{equation}
where $g =  \sum_{u\in \mathcal{V}_e} \frac{n_u}{n} \cdot g'_{u}$, and $\gamma$ denotes the learning rate.
Thus, the FL process with the networked control variates framework unfolds through the following series of steps:
\begin{itemize}
\item \textbf{Global parameters broadcasting}: The central server initializes the global model parameters $\theta_0$ and broadcasts them to all edge workers.
\item \textbf{Local model updating}: Upon receiving the global parameters, all workers update their local models and evaluate their performance using local data.
\item \textbf{Local-side gradient processing}: Each edge worker $u$ calculates the gradients $g^{i}_u$ of the local samples and reshapes them using the RLOO method to obtain $g'^{i}_u$.
\item \textbf{Communication of local gradients}: Each worker communicates the expectation gradient $g_{u} = \frac{1}{n_u} \sum_{i=1}^{n_u} g'^{i}_u$ to the central server.
\item \textbf{Server-side gradient processing}: The server applies the RLOO-based control variates method to further process each worker's gradient based on equation \eqref{server_cv}.
\item \textbf{Global model updating}: The server estimates the global gradient $g$ by aggregating the reshaped gradients $g'_{u}$. Then, the server updates the global model based on \eqref{NCV_update}.
\end{itemize}

\subsection{Double Control Variates Interpretation}
The derivation of the globally expected gradient $g$ from the local gradients $g_u^i$ can be understood through the application of a double control variates procedure, much like the method introduced in \citep{titsias2022double}. 
To express this:
\begin{equation}
g =\sum_{u\in\mathcal{V}_e} \frac{n_u}{n} \left( \frac{1}{n_u} \sum_{i=1}^{n_u} \left(g_u^{i} - \alpha c_{\mathcal{D}_u\setminus i}\right) - c_{\mathcal{V}_e \setminus u} \right),
\label{Double}
\end{equation}
where $c_{\mathcal{D}_u\setminus i}$ represents the leave-one-out control variate at each local worker node, and $c_{\mathcal{V}_e \setminus u} = \sum_{v \neq u} \frac{n_v}{n-n_u} \cdot g_{v}$ corresponds to the control variate at the server. 
Since $g_{v}^{i} = \frac{1}{n_v}\sum_{i=1}^{n_v} \left(g_v^{i} - c_{\mathcal{D}_v\setminus i} \right)$, we can recast equation \eqref{Double} as:
\begin{align}\small
g & =  \frac{1}{n} \sum_{i=1}^{n} \left(g^{i}_u - \alpha c_{\mathcal{D}_u^i\setminus i}\right) \notag \\
& - \frac{1}{n} \sum_{i=1}^{n} \left( \sum_{v \neq u} \frac{n_v}{n-n_u} \left( \frac{1}{n_v} \sum_{j \in \mathcal{D}_v} \left(g_v^{j} - \alpha c_{\mathcal{D}_v \setminus j} \right) \right) \right) \notag \\
& = \frac{1}{n} \sum_{i=1}^{n} \left(g^{i}_u - \alpha  c_{\mathcal{D}_u^i\setminus i} - \sum_{j\notin\mathcal{D}_u} \frac{g^{j}_v - \alpha c_{\mathcal{D}_v\setminus j}}{n - n_u}\right),
\label{NCV}
\end{align}
where $\mathcal{D}_v \neq \mathcal{D}_u^i, i\in\mathcal{D}_u^i$ and the terms $c_{\mathcal{D}_u\setminus i}$ and $c_{\mathcal{D}_v\setminus j}$ represent control variates that emulate the double control variate appearances as elaborated in \citep{titsias2022double}. This elucidates that the double control variates serve as a fundamental mathematical interpretation of the networked control variates within the paradigm of federated learning. The double control variates theory reveals why the networked control variates can effectively reduce the gradient variance.

\section{Algorithm Design}
In this section, we first give main theoretical results of our proposed solution and then present a practical algorithm design of the novel federated learning approach, which incorporates network control variates with the RLOO estimator.
\subsection{Main Theoretical Results}
We will reveal three theoretical results of our solution.
Firstly, we introduce the unbiasedness of the networked CV.
\begin{proposition}
    The gradient estimator, denoted by \eqref{NCV} and generated using the networked CV in federated learning with RLOO-based control variates \(c_{\mathcal{D}_v \setminus i}\), is unbiased.
\end{proposition}
See the supplementary material of \textbf{Appendix A} for the proof. An unbiased gradient estimator ensures that the expected value of the estimator is equal to the true gradient. In the context of our proof, this implies that our estimator \(g\), on average, perfectly captures the true gradient without any systematic deviation. The use of control variates in \(g^*\) helps reduce the variance of the estimator, which can be particularly beneficial in stochastic environments or when data is decentralized, as in FL. However, the true power of the control variate technique, as demonstrated, is that it achieves this reduction in variance without introducing bias, ensuring that \(g\) remains an unbiased estimator of the gradient.

\begin{proposition}
    For the FedNCV gradient estimator \eqref{NCV}, the optimal value of \( \alpha \) that minimizes the variance is given by
    \begin{tiny}
    \begin{equation}
        \alpha = \frac{2a^2 \left(\mathbb{E}\left[g^i_u c_{\mathcal{D}_u^i\setminus i} \right] + \mathbb{E}[g^i_u] - \frac{1}{a} \sum_{j\notin\mathcal{D}_u} E[g^{j}_v]\right) + \sum_{j\notin\mathcal{D}_u} \mathbb{E}\left[ g^{j}_v c_{\mathcal{D}v\setminus j} \right]}{2a^2 \mathbb{E}\left[ (c_{\mathcal{D}_u^i\setminus i})^2 \right] + \sum_{j\notin\mathcal{D}_u} \mathbb{E}\left[ (c_{\mathcal{D}_v\setminus j})^2 \right]}.
    \end{equation}
    \end{tiny}
\end{proposition}
The proof can be found in the supplementary material of \textbf{Appendix B}. The presented proposition outlines an optimal value for \( \alpha \) in the context of the FedNCV gradient estimator that minimizes variance. An explicit expression for \( \alpha \) is provided, which is based on specific combinations of expected values and gradients. The subsequent proof breaks down the variance of the estimator and determines the \( \alpha \) that minimizes this variance through differentiation. we now proceed to introduce the third Proposition, which theoretically elucidates the performance enhancement of networked CV over single control variate.

\begin{proposition}
    The variance of the networked control variate estimator \( h(\alpha) \) is lower than the variance of the RLOO-based control variate (single version) estimator \( h_s(\alpha) \).
\end{proposition}
Due to space constraints, we briefly discuss this in the supplementary material of \textbf{Appendix C}. In essence, this theorem underpins the theoretical basis for the effectiveness of employing networked control variates over pure RLOO-based control variates, especially when aiming for minimized gradient variances.

\subsection{Federated Learning Algorithm with Networked CV}
We give the pseudo-code (\textbf{Algorithm 1}) of the federated learning algorithm with networked control variates and describe it from both worker-side and server-side.
The algorithm operates in an iterative manner across a specified time set $\mathcal{T}$.
From a worker's perspective, an edge worker $u \in \mathcal{V}_e$ starts each iteration by receiving the current global parameters $\theta_t$ and loading local data $\mathcal{D}_u$. Using this data, the local models are updated, and gradients $g^{i}_u$ are calculated. The RLOO method is then used to reshape these gradients into $g'^{i}_u = g^{i}_u - \alpha_u \cdot c_{\mathcal{D}_u\setminus i}$. The final step at the worker level is to communicate the expected gradient, denoted as $g_{u}$, to the server.
The server-side processing begins with the server estimating the global gradient $g$ based on the received gradients from the workers. The global model parameters are then updated according to the equation $\theta_{t+1} = \theta_t - \gamma \cdot g$, where $\gamma$ is a nonnegative learning rate. 
In addition, server-side operations involve updating the hyperparameter $\bm{\alpha} = \{\alpha_u\}$.
To wrap up an iteration, the server broadcasts the updated parameters, namely $\bm{\alpha}$, and $\theta_{t+1}$, to all edge workers. 
\begin{algorithm}
	\caption{Federated Learning Algorithm with Networked Control Variates}
	\LinesNumbered
	\KwIn{Initial global model parameters $\theta_0$ and $\alpha$, and nonnegative learning rate $\gamma$.
	}
	\While{$t\in\mathcal{T}$}{
		\textcolor[rgb]{0.41568,0.52549,0.56078}{/*Worker-side processing*/}\\
		\ForEach{Edge Worker $u \in \mathcal{V}_e$}
		{
			Receive the global parameters $\theta_t$ and $\alpha_u$.\\
			Update local models and calculate the gradients $g^{i}_u$ using local data $\mathcal{D}_u$.\\
			Reshape gradients using the RLOO method: $g'^{i}_u = g^{i}_u - \alpha_u \cdot c_{\mathcal{D}_u\setminus i}$.\\
			Communicate the expected gradient to the server: $g_{u} = \sum_{i=1}^{n_u}p_{\theta}^{u,i} \cdot g'^{i}_u$.\\
		}
		\
		\textcolor[rgb]{0.41568,0.52549,0.56078}{/*Server-side processing*/}\\
		Estimate the global gradient $g$ based on \eqref{server_cv} and \eqref{Double}, then update the global model parameters: \\
        Adapt $\theta$: $\theta_{t+1} \gets \theta_t - \gamma \cdot g$.\\
        Update hyperparameters: $\alpha_u \gets \alpha_u - \gamma \cdot \Delta||g_u||_2^2$\\
	  Broadcast the global model parameters $\theta_{t+1}$ and $\bm{\alpha}$ to all edge workers.\\
	}
\end{algorithm}

\textbf{Distributed Implementation:}
The algorithm exhibits three distinct parts in its distributed implementation.
The initial stage involves worker-side processing where each edge worker $u \in \mathcal{V}_e$ receives the global parameters $\theta_t$ and $\bm{\alpha}$, updates local models with local data $\mathcal{D}_u$, calculates the gradients $g^{i}_u$, and reshapes them via RLOO.
The second phase focuses on server-side processing. In this stage, the server estimates the global gradient $g$, updates the global model parameters $\theta_{t+1}$ and $\bm{\alpha}$.
The final phase involves server-worker communication, as the updated local model and broadcasts the updated global model parameters, enabling parallel distributed learning and workers to adjust their local models for the next iteration.
It is important to note that the edge workers perform their operations using local data and only need to communicate with the server for parameter updates. Additionally, all communication and computation processes in the worker-side processing phase can be carried out independently and concurrently, demonstrating the distributed nature of the algorithm.
In terms of complexity, the algorithm is dictated by the number of edge workers $\mathcal{V}_e$ and the number of data points in their respective local datasets $\mathcal{D}_u$. Therefore, the space complexity and time complexity of the algorithm can be denoted as $\mathcal{O}(|\mathcal{V}_e|)$ and $\mathcal{O}(|\mathcal{D}_u|)$ respectively, where $|\mathcal{V}_e|$ represents the number of edge workers and $|\mathcal{D}_u|$ denotes the size of local datasets.

\begin{table*}[]
\centering
\small
\begin{tabular}{|c|cc|cc|cc|cc|}
\hline
\multirow{2}{*}{Solutions}
& \multicolumn{2}{c|}{CIFAR-10 (10 classes)}
& \multicolumn{2}{c|}{CIFAR-100 (100 classes)}
& \multicolumn{2}{c|}{Tiny-ImageNet (200 classes)}
& \multicolumn{2}{c|}{EMNIST (62 classes)}
\\ \cline{2-9} 
& \multicolumn{1}{c|}{test before} & test after       
& \multicolumn{1}{c|}{test before} & test after       
& \multicolumn{1}{c|}{test before} & test after       
& \multicolumn{1}{c|}{test before} & test after
\\ \hline
FedAvg
& \multicolumn{1}{c|}{\underline{48.96}(1.37)} & \underline{49.07}(0.61)
& \multicolumn{1}{c|}{8.46(1.42)} & \underline{39.71}(2.76) 
& \multicolumn{1}{c|}{7.64(0.88)} & \underline{29.73}(1.71) 
& \multicolumn{1}{c|}{74.25(5.01)} & \underline{95.87}(1.57)
\\ \hline
FedProx
& \multicolumn{1}{c|}{47.00(1.37)} & 47.47(1.79)
& \multicolumn{1}{c|}{5.80(0.81)}  & 32.88(1.77)
& \multicolumn{1}{c|}{6.97(0.89)}  & 27.23(6.97)
& \multicolumn{1}{c|}{71.86(6.71)} & 93.24(1.57)    
\\ \hline
SCAFFOLD
& \multicolumn{1}{c|}{35.53(1.47)} & 36.40(2.15)
& \multicolumn{1}{c|}{15.25(1.44)} & 15.36(1.50) 
& \multicolumn{1}{c|}{5.42(1.46)}  & 5.97(1.79)
& \multicolumn{1}{c|}{37.20(8.57)} & 43.04(11.75) 
\\ \hline
FedRep
& \multicolumn{1}{c|}{27.95(1.94)} & 33.54(2.18)
& \multicolumn{1}{c|}{25.82(1.98)} & 26.15(2.03)
& \multicolumn{1}{c|}{19.56(1.70)} & 22.14(1.43)
& \multicolumn{1}{c|}{87.20(1.20)} & 88.43(1.27)
\\ \hline
FedPer
& \multicolumn{1}{c|}{45.22(2.07)} & 46.48(1.25) 
& \multicolumn{1}{c|}{31.61(1.96)} & 35.19(1.51) 
& \multicolumn{1}{c|}{24.27(1.55)} & 26.63(1.53) 
& \multicolumn{1}{c|}{90.21(0.84)} & 91.65(0.77)
\\ \hline
pFedSim
& \multicolumn{1}{c|}{44.34 (1.99)}            & 45.31(1.03) 
& \multicolumn{1}{c|}{\underline{35.32}(2.09)} & 36.46(1.97) 
& \multicolumn{1}{c|}{\underline{28.34}(1.57)} & 29.01(1.51) 
& \multicolumn{1}{c|}{\underline{91.98}(1.29)} & 92.96(0.68)
\\ \hline
FedNCV
& \multicolumn{1}{c|}{\textbf{53.49}(1.86)} & \textbf{55.79}(0.92) 
& \multicolumn{1}{c|}{\textbf{44.41}(2.44)} & \textbf{45.29}(2.05) 
& \multicolumn{1}{c|}{\textbf{32.54}(1.53)} & \textbf{33.06}(1.53) 
& \multicolumn{1}{c|}{\textbf{95.34}(1.59)} & \textbf{96.35}(1.31)
\\ \hline
\end{tabular}
\vspace{-0.5em}
\caption{Summary of average model accuracy across datasets, presented as mean (standard deviation), for 100 iterations. The best and second-best results are highlighted with bold and underlined, respectively.}
\label{table_1}
\vspace{-1em}
\end{table*}

\section{Related Works}
Research in Federated Learning (FL) primarily aims to reduce communication overhead and speed up model convergence. 
Efforts to address slow convergence caused by non-IID data distribution across clients encompass techniques such as adaptive local training \citep{li2020federated}, differential node weighting \citep{wu2021fast}, strategic node selection \citep{wu2022node}, shared representation \citep{collins2021exploiting} and gradient refinement \citep{karimireddy2020scaffold,arivazhagan2019federated,tan2023pfedsim}. Notably, the FedProx algorithm \citep{li2020federated} adds a regularization term to FL. The work by \citep{wu2021fast} presents FedAdp, emphasizing varied weights for node contributions, and \citep{wu2022node} suggests a probabilistic approach for optimal node aggregation. Meanwhile, SCAFFOLD \citep{karimireddy2020scaffold} deploys control variates to lessen the impact of "client-side bias."
In the realm of personalization, FedRep \citep{collins2021exploiting} exploits shared representations model for personalized FL, while FedPer \citep{arivazhagan2019federated} champions a dual-layer technique for enhanced customization amidst data diversity. Further, pFedSim \citep{tan2023pfedsim} designs a split network approach, where one segment focuses on personalized feature extraction  and the other on local classification model.

In gradient-based optimization, Control Variates (CVs) method \citep{bengio2013estimating,mnih2014neural,ranganath2014black,weaver2013optimal} stands as an effective strategy for reducing gradient variance. Recent advancements include the RLOO techniques \citep{kool2020estimating,mnih2016variational,salimans2014using}, as well as double control variates and sample-specific baselines of gradient estimator \citep{grathwohl2017backpropagation,paisley2012variational,titsias2022double,tucker2017rebar}.

Compared to existing work, our approach presents several notable distinctions and improvements:
First, unlike SCAFFOLD \citep{karimireddy2020scaffold} that adjusts control variates based on discrepancies between local and global model parameters, we employ control variates in both the server and client side, signifying a fundamental shift in the application of control variates.
Second, local control variates effectively reduce gradient variance, enhancing personalization, while global control variates increase the stability of aggregated gradients for large-scale applications.
This innovative framework enhances the scalability of FL, accelerates convergence, and improves overall efficiency.

\begin{figure*}[!t]
	\begin{center}
		\includegraphics[width=\linewidth]{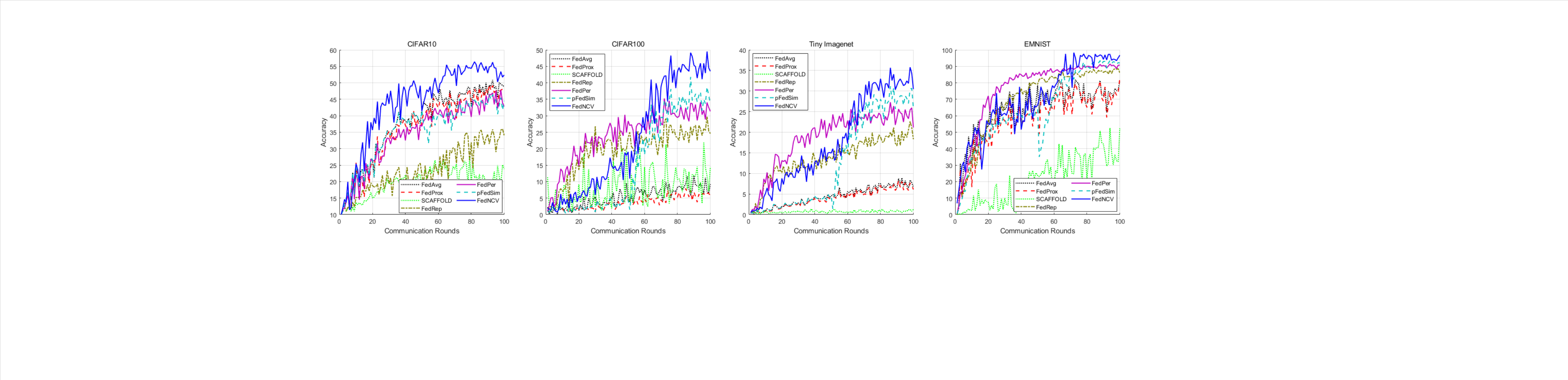}
	\end{center}
	\begin{center}
		\vspace{-1em}
		\caption{Performance evaluation of seven solutions: pre-test accuracy vs. communication rounds across four datasets.}
		\label{Fig1}
	\end{center}
	\vspace{-3em}
\end{figure*}

\begin{figure}[!t]
	\begin{center}
		\includegraphics[width=\linewidth]{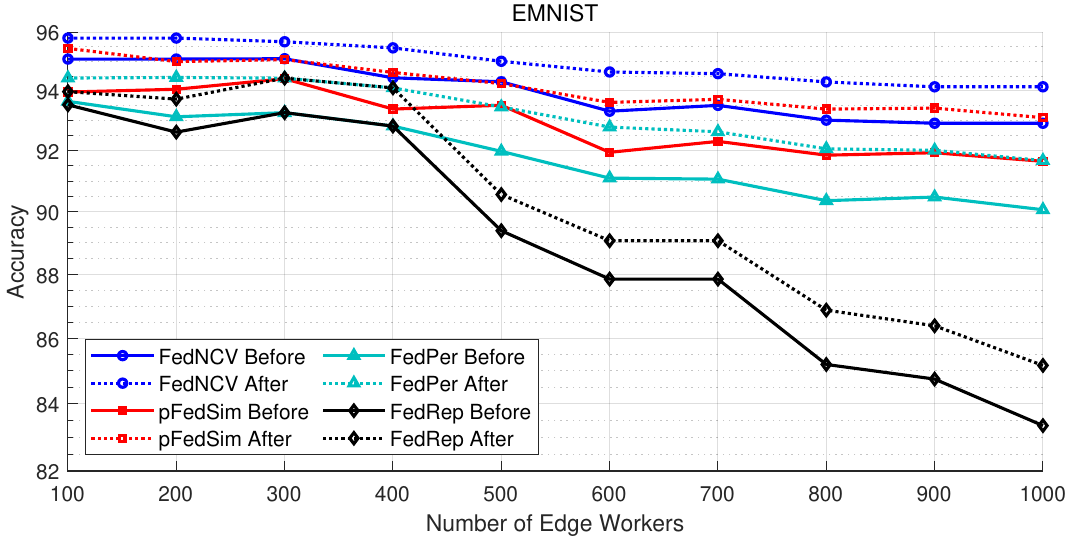}
	\end{center}
	\begin{center}
		\vspace{-1em}
		\caption{Scalability performance of four solutions across varying numbers of edge workers.}
		\label{Fig2}
	\end{center}
	\vspace{-3em}
\end{figure}

\section{Performance Evaluation}
In this section, we first delineate our experimental setup, followed by a comprehensive analysis of the results. The code is available in the supplementary material.

\subsection{Experiment Setup}
We assessed four key datasets: CIFAR-10 and CIFAR-100 \citep{krizhevsky2009learning}, Tiny-ImageNet \citep{le2015tiny}, and EMNIST \citep{cohen2017emnist}. 
Non-IID data distribution presents a critical challenge in FL. To emulate such heterogeneity, we utilized the Dirichlet distribution, a well-accepted method for simulating non-IID patterns \citep{tan2023pfedsim}.
Specifically, each dataset was partitioned into subsets by using this distribution. 
For our experiments, we set the hyperparameter of the Dirichlet distribution \( \alpha = 0.1\). This setting accentuates the non-IID nature, suggesting that the dataset held by a single client might not necessarily encompass all classes, expressed formally as \( |Y_i | \leq |Y| \), where \( Y_i \) indicates the label space of data on the i-th client \citep{tan2023pfedsim}.
Our evaluation began with an in-depth analysis at a 100-client scale, later expanding to assess performance from 100 to 1,000 clients for FedNCV.
All experiments were performed on a workstation (i9-12900K, 64G, RTX 3090).
To evaluate the performance of the algorithms with limited communication, we set the maximum number of global updates for all algorithms to 100.
For more details of the experiment settings, please refer to \textbf{Appendix D}.

We evaluated our FedNCV approach with an ensemble of six state-of-the-art FL solutions.
To gauge both generalization and personalization capabilities, we incorporated:
1) \textbf{FedAvg} \citep{mcmahan2017communication}: Serving as our standard benchmark, it predominantly assesses generalization performance;
For a more comprehensive comparison, we included:
2) \textbf{FedProx} \citep{li2020federated} integrates a proximal term into the loss function, thereby ensuring that local models closely mirror the global archetype;
3) \textbf{SCAFFOLD} \citep{karimireddy2020scaffold}: This technique harnesses control variates, effectively mitigating ``client-drift'' in local updates and solidifying gradient stability;
4) \textbf{FedRep} \citep{collins2021exploiting} offers shared data representations model in personalized FL
5) \textbf{FedPer} \citep{arivazhagan2019federated} employs a ``base + personalization layer'' strategy for model training.
6) \textbf{pFedSim} \citep{tan2023pfedsim} designs a split network approach, where one segment focuses on personalized feature extraction and the other on local classification.

\subsection{Evaluation Results}

In the CIFAR-10 dataset, our method, FedNCV, outperforms all the other algorithms, achieving an accuracy of 53.49\% (1.86) before testing and 55.79\% after testing. This result is especially promising, considering the challenges associated with CIFAR-10's relatively simple structure and limited classes. The performance improvement between the pre and post-test stages indicates FedNCV's robustness and adaptability.
FedNCV demonstrates strong efficacy on the CIFAR-100 dataset as well, securing the best pre-test accuracy of 44.41\% (2.44) and post-test accuracy of 45.29\% (2.05). The complexity of the CIFAR-100 dataset, with a higher number of classes, validates the versatility of FedNCV. The consistent performance portrays the algorithm's potential to handle varied data distributions without significant degradation in accuracy.
In the Tiny-ImageNet dataset, our method, FedNCV, achieves an accuracy of 32.54\% (1.53) before testing and 33.06\% (1.53) after testing across 200 classes. This performance is the best among the considered solutions, with FedNCV registering the highest accuracy for both pre-testing and post-testing. While other methods such as pFedSim, FedPer, and FedRep exhibit some degree of competence, they lag behind FedNCV. Specifically, the SCAFFOLD solution shows the least improvement, demonstrating an accuracy of only 5.42\% (1.46) before testing and 5.97\% (1.79) after testing. The results from FedNCV underline its capability to manage more intricate data structures across 200 classes. 
The results on the EMNIST dataset further consolidate FedNCV's supremacy. Achieving the top accuracy of 95.34\% (1.59) before testing and 96.35\% (1.31) after testing showcases FedNCV's superior handling of the handwritten recognition task. The increase in accuracy from pre to post-test underlines the effectiveness of the learning process.

FedNCV demonstrates an outstanding and consistent performance across various datasets with different complexities and class numbers.
The improvements in accuracy from the pre-test to post-test indicates potential areas for optimization in handling complex datasets and stages reflect the method's robustness, scalability, and adaptability to different learning scenarios. 
These findings reinforce FedNCV as a promising solution in the federated learning landscape.

Figure \ref{Fig1} displays the comparative results of our method against six other methods across four datasets. In the CIFAR-10 dataset, our method outperforms all other methods throughout almost the entire iterative process. Notable performers like FedAvg, FedProx, FedPer, and pFedSim demonstrate decent results, while FedRep and SCAFFOLD perform the worst. In the CIFAR-100 scenario, FedRep and FedPer surpass our method during the first 60 rounds of iteration, but our method then exceeds all others in subsequent rounds. FedAvg, FedProx, and SCAFFOLD, however, show slow improvement with each iteration. In the Tiny-ImageNet dataset, FedPer is superior to FedNCV in the first 60 rounds,  but subsequently, it is surpassed by our method at round 100. In the EMNIST dataset, our method outperforms all other solutions after 70 rounds of iteration.
Additionally, FedAvg, FedProx, and SCAFFOLD struggle to converge in the Tiny-ImageNet and EMNIST datasets, whereas FedRep, FedPer, and pFedSim exhibit better performance.

We further evaluated the scalability performance of FedNCV in comparison to FedRep, FedPer, and pFedSim methods across different ranges of edge workers in EMNIST dataset. It is important to note that our comparison included FedAvg, FedProx, and SCAFFOLD, as these classic methods are known to have relatively inferior performance in terms of scalability. Figure \ref{Fig2} presents the accuracy performance both pre and post-test the completion of 100 rounds, as the number of workers increases. It can be observed that FedNCV demonstrates commendable scalability.
An increase in the number of clients typically implies a greater diversity in data distributions, which may complicate the model's ability to converge to a globally optimal solution that is well-suited for all data distributions. As the user data increases from 100 to 1000, the accuracy decline in both pre and post-stages is only 1.66\% and 2.17\%, respectively. In contrast, the other three methods record a more significant decrease: pFedSim at 2.31\% and 2.35\%, FedPer at 3.58\% and 2.75\%, and FedRep at 10.18\% and 8.80\%. This experiment underlines the robust scalability of FedNCV compared to other tested methods, particularly when scaling across a more extensive range of edge workers.

\section{Conclusion and Future Work}
In this paper, we presented the FedNCV framework tailored for the FL paradigm. Our investigations spotlighted the challenges intrinsic to FL, particularly the gradient variance arising from heterogeneous client data distributions. FedNCV, with its foundation on the RLOO control variate, offers a solution by managing gradient variance both at the client and server levels. Its dual-side application, formalized through a composite of multiple control variates, makes it robust against the prevalent challenges in FL, especially data heterogeneity and scalability.
In addition, we provide a mathematical expression that captures this integration of composite control variates in FedNCV and present three theoretical results with corresponding proofs.
Our empirical analysis on four diverse datasets under a Dirichlet distribution demonstrated FedNCV's effectiveness. Benchmarked against six SOTA methods, our approach showcased its superior performance, signifying its potential for large-scale applications.

However, our study is not without limitations, and we envision several avenues for future research:
Though FedNCV is anchored on RLOO, integrating diverse control variate techniques may further enhance its performance and efficiency.
Expanding FedNCV to accommodate different distributed node interactions, especially swarm learning \citep{warnat2021swarm}, would boost its scope in FL and tackle specific decentralized learning challenges.
In essence, FedNCV signifies a pivotal advancement in managing FL intricacies, with ample opportunities for refinement in subsequent research endeavors.

\defaultbibliographystyle{unsrtnat}  
\putbib[aaai24]

\newpage

\onecolumn

\appendix

\newpage
\end{bibunit}

\begin{bibunit}
\section{{\large Taming Gradient Variance in Federated Learning with Networked Control Variates}}
\vspace{1.5em}
\begin{center}
\end{center}
\vspace{1.5em}

\subsection{Appendix A: Proof of the unbiasedness of networked control variates}

\begin{proposition}
    The gradient estimator, denoted by \eqref{NCV} and generated using the networked CV in federated learning with RLOO-based control variates \(c_{\mathcal{D}_v \setminus i}\), is unbiased.
\end{proposition}

\begin{proof}
    To demonstrate the unbiased nature of the gradient estimator \eqref{NCV}, consider the expectation gradient:
    \begin{align}
        g = \frac{1}{n} \sum_{i=1}^{n} \left( g^{i}_u - \alpha c_{\mathcal{D}_u^i\setminus i} - \frac{\sum_{j\notin\mathcal{D}_u} (g^{j}_v - \alpha c_{\mathcal{D}_v\setminus j})}{n - n_u} \right),
        \label{NCV_1}
    \end{align}
    where \(c_{\mathcal{D}_u^i\setminus i} =  \frac{1}{n_u - 1} \sum_{j \neq i} g^{j}_u\). Define \(g^*\) as the gradient estimator that employs control variates once:
    \begin{align*}
    g^* &= \sum_{u\in\mathcal{V}_e} \frac{n_u}{n} g_u = \frac{1}{n} \sum_{i=1}^{n} g'^{i}_u \\
    &= \frac{1}{n} \left( \sum_{j\notin\mathcal{D}_u} \left(g^{j}_v - \alpha c_{\mathcal{D}_v\setminus j} \right) + \sum_{i\in\mathcal{D}_u} \left(g^{i}_u - \alpha c_{\mathcal{D}_u\setminus i} \right) \right).
    \end{align*}
    Rewriting \eqref{NCV_1} using \(g^*\), we get:
    \begin{align}
        g & =  \frac{1}{n} \sum_{i=1}^{n} \left(g^{i}_u - \alpha  c_{\mathcal{D}_u\setminus i} - \frac{ ng^* -\sum_{j\in\mathcal{D}_u^i} \left(g^{j}_u -  \alpha c_{\mathcal{D}_u\setminus j} \right)}{n-n_u} \right) \nonumber\\
        & = \frac{1}{n} \sum_{i=1}^{n} \left(g^{i}_u - \alpha c_{\mathcal{D}_u\setminus i} \right) - \frac{n g^*}{n-n_u}  \nonumber \\
        & \quad +  \frac{\frac{1}{n} \sum_{i=1}^{n} \sum_{j\in\mathcal{D}_u^i} \left(g^{j}_u -  \alpha c_{\mathcal{D}_u\setminus j} \right)}{n-n_u} \nonumber \\
        & = \frac{1}{n} \sum_{i=1}^{n} \left(g^{i}_u - \alpha c_{\mathcal{D}_u\setminus i} \right) - \frac{n g^*}{n-n_u} \nonumber \\
        & \quad +  \frac{\frac{1}{n} \sum_{u\in\mathcal{V}_e} \sum_{i\in\mathcal{D}_u} \sum_{j\in\mathcal{D}_u^i} \left(g^{j}_u -  \alpha c_{\mathcal{D}_u\setminus j} \right)}{n-n_u} \nonumber \\
        & = \frac{1}{n} \sum_{i=1}^{n} \left(g^{i}_u - \alpha c_{\mathcal{D}_u\setminus i} \right) - \frac{n g^*}{n-n_u} \nonumber \\
        & \quad +  \frac{\frac{1}{n} \sum_{u\in\mathcal{V}_e} \sum_{i=1}^{n} n_u \left(g^{i}_u -  \alpha c_{\mathcal{D}_u\setminus i} \right)}{n-n_u} \nonumber \\
        & = \frac{1}{n} \sum_{i=1}^{n} \left(g^{i}_u - \alpha c_{\mathcal{D}_u\setminus i} \right) - \frac{n g^*}{n-n_u}  +  \frac{\sum_{u\in\mathcal{V}_e}n_u g^*}{n-n_u} \nonumber \\
        & = \frac{1}{n} \sum_{i=1}^{n} \left(g^{i}_u -  \alpha c_{\mathcal{D}_u\setminus i} \right).
    \end{align}
    Given the expectation of \(g\), we observe that the contribution from control variates cancels out. Likewise, the term \(c_{\mathcal{D}_u\setminus i}\) has a net contribution of zero and can thus be omitted. Consequently, the gradient estimator is essentially the mean of the \(g^{i}_u\) gradients, ensuring its unbiasedness relative to the genuine gradients.
\end{proof}

\subsection{Appendix B: Proof of optimal value for $\alpha$ in the context of the FedNCV}
\begin{proposition}
    For the NCV gradient estimator \eqref{NCV}, the optimal value of \( \alpha \) that minimizes the variance is given by
    \begin{equation}
        \alpha = \frac{2a^2 \left(\mathbb{E}\left[g^i_u c_{\mathcal{D}_u^i\setminus i} \right] + \mathbb{E}[g^i_u] - \frac{1}{a} \sum_{j\notin\mathcal{D}_u} E[g^{j}_v]\right) + \sum_{j\notin\mathcal{D}_u} \mathbb{E}\left[ g^{j}_v c_{\mathcal{D}v\setminus j} \right]}{2a^2 \mathbb{E}\left[ (c_{\mathcal{D}_u^i\setminus i})^2 \right] + \sum_{j\notin\mathcal{D}_u} \mathbb{E}\left[ (c_{\mathcal{D}_v\setminus j})^2 \right]}.
    \end{equation}
\end{proposition}

\begin{proof}
    Given the expression:
    \[ h(\alpha) = g^i_u - \alpha c_{\mathcal{D}_u^i\setminus i} - \frac{\sum_{j\notin\mathcal{D}_u} (g^{j}_v - \alpha c_{\mathcal{D}_v\setminus j})}{n - n_u}, \]
    let's break down the variance:
    \[ \text{Var}[h(\alpha)] = \mathbb{E}[h(\alpha)^2] - \mathbb{E}[h(\alpha)]^2 \]
    Expanding \( \mathbb{E}[h(\alpha)^2] \) and $\mathbb{E}[h(\alpha)]^2$:
    \begin{align} 
        \mathbb{E}[h(\alpha)^2] &= \mathbb{E}[(g^i_u)^2] + \alpha^2 \mathbb{E}[(c_{\mathcal{D}_u^i\setminus i})^2] - 2\alpha E[g^i_u c_{\mathcal{D}_u^i\setminus i}]  \notag\\ 
        & + \frac{1}{(n - n_u)^2} \mathbb{E}[(\sum_{j\notin\mathcal{D}_u} (g^{j}_v - \alpha c_{\mathcal{D}_v\setminus j}))^2] \notag\\
        & - 2\frac{\mathbb{E}[g^i_u (\sum_{j\notin\mathcal{D}_u} (g^{j}_v - \alpha c_{\mathcal{D}_v\setminus j}))]}{n - n_u}  \notag\\
        & + 2\alpha\frac{\mathbb{E}[c_{\mathcal{D}_u^i\setminus i} (\sum_{j\notin\mathcal{D}_u} (g^{j}_v - \alpha c_{\mathcal{D}_v\setminus j}))]}{n - n_u}. \notag \\
        \mathbb{E}[h(\alpha)] = & \mathbb{E}[g^i_u] - \alpha \mathbb{E}[c_{\mathcal{D}_u^i\setminus i}] - \frac{\sum_{j\notin\mathcal{D}_u} E[g^{j}_v] - \alpha \mathbb{E}[c_{\mathcal{D}_v\setminus j}]}{n - n_u}. \notag 
    \end{align}
    The terms \( 2\alpha\frac{\mathbb{E}[c_{\mathcal{D}_u^i\setminus i} (\sum_{j\notin\mathcal{D}_u} (g^{j}_v - \alpha c_{\mathcal{D}_v\setminus j}))]}{n - n_u} \) and $2\frac{E[g^i_u (\sum_{j\notin\mathcal{D}_u} (g^{j}_v - \alpha c_{\mathcal{D}_v\setminus j}))]}{n - n_u}$ can be eliminated since $g_u^i$ and \( c_{\mathcal{D}_u^i\setminus i} \) are independent of \( \sum_{j\notin\mathcal{D}_u} (g^{j}_v - \alpha c_{\mathcal{D}_v\setminus j}) \). Then, We can simplify the above terms and obtain Var$[h(\alpha)]$ as: 
    \begin{align}
        \text{Var} [h(\alpha)] 
        & = \mathbb{E}\left[ (g^i_u)^2 \right] + \alpha^2 \mathbb{E}\left[ (c_{\mathcal{D}_u^i\setminus i})^2 \right] - 2\alpha \mathbb{E}\left[ g^i_u c_{\mathcal{D}_u^i\setminus i} \right] \notag \\
        & + \frac{1}{(n - n_u)^2} \mathbb{E}\left[ (\sum_{j\notin\mathcal{D}_u} (g^{j}_v - \alpha c_{\mathcal{D}_v\setminus j}))^2 \right] \notag \\
        & - \left( \mathbb{E}[g^i_u] - \frac{\sum_{j\notin\mathcal{D}_u} E[g^{j}_v]}{n - n_u} \right)^2 
        \label{double_cv_1}
    \end{align}
    The goal is to minimize \(\text{Var}[h(\alpha)]\). To find the optimal \(\alpha\), differentiate the above expression with respect to \(\alpha\):
    \begin{align}
        \frac{d}{d\alpha} & \text{Var}[h(\alpha)] = 2\alpha \mathbb{E}\left[ (c_{\mathcal{D}_u^i\setminus i})^2 \right] - 2E\left[g^i_u c_{\mathcal{D}_u^i\setminus i} \right] \notag \\ 
        & - \frac{2}{(n - n_u)^2} \sum_{j\notin\mathcal{D}_u} \mathbb{E}\left[ g^{j}_v c_{\mathcal{D}_v\setminus j} \right] \notag \\
        & + \frac{2\alpha}{(n - n_u)^2} \sum_{j\notin\mathcal{D}_u} \mathbb{E}\left[ (c_{\mathcal{D}_v\setminus j})^2 \right] \notag \\
        & - 2\mathbb{E}\left[g^i_u\right] + 2\alpha \mathbb{E}[c_{\mathcal{D}_u^i\setminus i}]  \notag \\
        & + \frac{2}{(n - n_u)} \sum_{j\notin\mathcal{D}_u} \mathbb{E}[g^{j}_v] - \frac{2\alpha}{(n - n_u)} \sum_{j\notin\mathcal{D}u} \mathbb{E}[c_{\mathcal{D}_v\setminus j}]. \notag
    \end{align}
    Since $\mathbb{E}[c_{\mathcal{D}_v\setminus j}] = 0$, we can omit the expected terms and get:
    \begin{align}
        \frac{d}{d\alpha} & \text{Var}[h(\alpha)] = 2\alpha \mathbb{E}\left[ (c_{\mathcal{D}_u^i\setminus i})^2 \right] - 2\mathbb{E}\left[g^i_u c_{\mathcal{D}_u^i\setminus i} \right] \notag \\ 
        & - \frac{2}{(n - n_u)^2} \sum_{j\notin\mathcal{D}_u} \mathbb{E}\left[ g^{j}_v c_{\mathcal{D}_v\setminus j} \right] \notag \\
        & + \frac{2\alpha}{(n - n_u)^2} \sum_{j\notin\mathcal{D}_u} \mathbb{E}\left[ (c_{\mathcal{D}_v\setminus j})^2 \right] \notag \\
        & - 2\mathbb{E}\left[g^i_u\right] + \frac{2}{(n - n_u)} \sum_{j\notin\mathcal{D}_u} \mathbb{E}[g^{j}_v]. \notag
    \end{align}
    To minimize the variance of \( h(\alpha) \), differentiate with respect to \( \alpha \) and set the derivative to zero.
    From the above equation, the optimal value of \( \alpha \) to minimize the variance is:
    \begin{equation*}
        \alpha = \frac{2a^2 \left(\mathbb{E}\left[g^i_u c_{\mathcal{D}_u^i\setminus i} \right] + \mathbb{E}\left[g^i_u\right] - \frac{1}{a} \sum_{j\notin\mathcal{D}_u} \mathbb{E}\left[g^{j}_v\right]\right) + \sum_{j\notin\mathcal{D}_u} \mathbb{E}\left[ g^{j}_v c_{\mathcal{D}v\setminus j} \right]}{2a^2 \mathbb{E} \left[ (c_{\mathcal{D}_u^i\setminus i})^2 \right] + \sum_{j\notin\mathcal{D}_u} \mathbb{E}\left[ (c_{\mathcal{D}_v\setminus j})^2 \right]}.
    \end{equation*}
    Here, $a = (n - n_u)$. Hence, from the derived expression, we have the optimal value of \( \alpha \) which minimizes the variance. This confirms and concludes the proof of \textbf{Proposition 2}.
\end{proof}

\subsection{Appendix C: Improvement of networked control variates compared to pure control variates}
\begin{proposition}
    The variance of the networked control variate estimator \( h(\alpha) \) is lower than the variance of the original RLOO-based control variate (single version) estimator \( h_s(\alpha) \).
\end{proposition}
\begin{proof}
    As presented, the RLOO-based control variate is:
    \[ h_s(\alpha) = g^i_u - \alpha c_{\mathcal{D}_u^i\setminus i}, \]
    leading to its variance being expressed as:
    \[ \text{Var}[h_s(\alpha))] = \text{Var}[g^{i}_u] + \alpha^2 \text{Var}[c_{\mathcal{D}_u^i\setminus i}] - 2\alpha \text{Cov}\left[g^{i}_u, c_{\mathcal{D}_u^i\setminus i}\right]. \]
    With the zero-mean property of the control variate, \(c_{\mathcal{D}_u^i\setminus i}\), and the fundamental properties of variance and covariance, the variance for the single control variate is then:
    \begin{align*}
        \text{Var}[h_s(\alpha)] &= \mathbb{E}\left[(g^i_u)^2\right] - \left(\mathbb{E}\left[g^i_u\right]\right)^2 \notag \\
        & + \alpha^2 \left( \mathbb{E}\left[(c_{\mathcal{D}_u^i\setminus i})^2\right]  \right) - 2 \alpha \left( \mathbb{E}\left[g^i_u c_{\mathcal{D}_u^i\setminus i}\right] \right). \notag
    \end{align*}
    To contrast this with the double control variate, the difference between their variances is:
    \begin{align*}
        & \text{Var}[h(\alpha)] - \text{Var}[h_s(\alpha)] = \\ 
        & \frac{1}{(n - n_u)^2} \mathbb{E}\left[ \left(\sum_{j\notin\mathcal{D}_u} (g^{j}_v - \alpha c_{\mathcal{D}_v\setminus j})\right)^2 \right] \\
        & + 2 \frac{\mathbb{E}\left[g^i_u\right] \sum_{j\notin\mathcal{D}_u} \mathbb{E}\left[g^{j}_v\right]}{n - n_u} - \frac{\left(\sum_{j\notin\mathcal{D}_u} \mathbb{E}\left[g^{j}_v\right]\right)^2}{(n - n_u)^2}.
    \end{align*}
    This demonstrates the superior performance of the networked control variates over the single version in reducing variance, under the specified conditions.
\end{proof} 

\subsection{Appendix D: Detailed configurations of dataset and experiment settings}

\textbf{Datasets and Experimental Settings:} We utilized the benchmark dataset with the same training/testing split as employed in previous works. These datasets include MNIST \citep{lecun1998gradient}, CIFAR-10 and CIFAR-100 \citep{krizhevsky2009learning}, SVHN \citep{netzer2011reading}, Tiny-ImageNet \citep{le2015tiny}, EMNIST \citep{cohen2017emnist}, FMNIST \citep{paisley2012variational}, and CINIC-10 \citep{darlow2018cinic}.The basic information regarding the datasets is displayed in Table \ref{table 2}. For our evaluations, we employ the LeNet5 \citep{lecun2015lenet} model, which has demonstrated effective performance on various image classification tasks. Interested readers can refer to the supplementary materials for the code, which has been modified based on the original repository\footnote{https://github.com/KarhouTam/FL-bench} to adapt to our experimental setup. 

We further evaluated the accuracy performance of our FedNCV method in comparison to eleven other previous methods on the MNIST, SVHN, FMNIST, and CINIC-10 datasets.The details of the remaining eleven FL solutions are as follows:
\begin{itemize}
    \item \textbf{FedAvg} \citep{mcmahan2017communication}: Serving as our standard benchmark, it predominantly assesses generalization performance.
    \item \textbf{FedAvgM} \citep{hsu2019measuring}: By examining datasets with varying degrees of similarity, this method delves into the nuances of the Federated Averaging algorithm's efficacy and underscores the advantage of leveraging server momentum.
    \item \textbf{FedProx} \citep{li2020federated}: Integrates a proximal term into the loss function, thereby ensuring that local models closely mirror the global archetype.
    \item \textbf{SCAFFOLD} \citep{karimireddy2020scaffold}: This technique harnesses control variates, effectively mitigating ``client-drift'' in local updates and solidifying gradient stability.
    \item \textbf{FedDyn} \citep{acar2020federated}: Tackling the incongruence between local and global empirical losses, this solution propounds device-specific dynamic regularizers to bridge the observed gaps.
    \item \textbf{FedGen} \citep{zhu2021data}: This FL method champions a feature generator, augmenting local training via synthesized virtual features.
    \item \textbf{FedLC} \citep{zhang2022federated}: Tailored for FL, it recalibrates logits in accordance with class occurrence probabilities, optimizing softmax cross-entropy results.
    \item \textbf{MOON} \citep{li2021model}: An acronym for Model-Contrastive Federated Learning, MOON exploits the affinity between model representations, fortifying local training through model-level contrastive learning.
    \item \textbf{FedRep} \citep{collins2021exploiting}: Offers shared data representations model in personalized FL.
    \item \textbf{FedPer} \citep{arivazhagan2019federated}: Employs a ``base + personalization layer'' strategy for model training.
    \item \textbf{pFedSim} \citep{tan2023pfedsim}: Designs a split network approach, where one segment focuses on personalized feature extraction and the other on local classification.
\end{itemize}

\begin{table}
\centering
\caption{Summary of Datasets}
\label{table 2}
\begin{tabular}
{|l|m{5cm}|c|m{2cm}|m{2cm}|c|}
\hline
\textbf{Dataset Name} & \textbf{Description} & \textbf{Number of Classes} & \textbf{Number of Images} & \textbf{Resolution} & \textbf{Year} \\
\hline
MNIST & Handwritten digits & 10 & 70,000 & 28x28 & 1998 \\
\hline
CIFAR-10/100 & Images in 10 classes including airplanes, cars, birds, etc. & 10/100 & 60,000 & 32x32 & 2009 \\
\hline
SVHN & Street view house numbers & 10 & 600,000 & 32x32 & 2011 \\
\hline
Tiny-ImageNet & Tiny ImageNet dataset & 200 & 110,000 & 64×64 & 2015 \\
\hline
EMNIST & Extended MNIST with letters and digits & 62 & 805,263 & 28x28 & 2017 \\
\hline
FMNIST & Fashion item images & 10 & 70,000 & 28x28 & 2017 \\
\hline
CINIC-10 & Images in 10 classes including airplanes, cars, birds, etc. & 10 & 270,000 & 32x32 & 2018 \\
\hline
\end{tabular}
\end{table}

\subsubsection{Experimental Results:}

Figure \ref{Fig3} illustrates the evaluation results of the pre-test accuracy versus communication rounds across four datasets for our method, compared to eleven other approaches. Within these datasets, FedRep, FedPer, and our FedNCV are predominantly superior to other methods, essentially achieving optimal and suboptimal states as they approach equilibrium.
In the MNIST dataset, our approach is capable of achieving optimal performance around the 60th round, being on par with the initially most effective method, FedPer. Prior to the 60th round, our method consistently ranks within the top four, and post the 60th round, it mainly occupies the first or second position.
For the FMNIST dataset, our approach oscillates near the optimal value around the 60th round, and its performance subsequently stabilizes in either the first or second position.
In the SVHN dataset, FedRep and FedPer appear to be the most effective methods before the 60th round. However, following the 60th round, our method's accuracy dramatically increases, stabilizing mainly in the first or second position.
Similarly, in the CINIC-10 dataset, pFedSim, FedRep, FedPer, and our FedNCV are conspicuously ahead of other methods throughout all rounds. Post the 70th round, our approach essentially stabilizes within the top three positions.

\begin{figure*}[!t]
	\begin{center}
		\includegraphics[width=\linewidth]{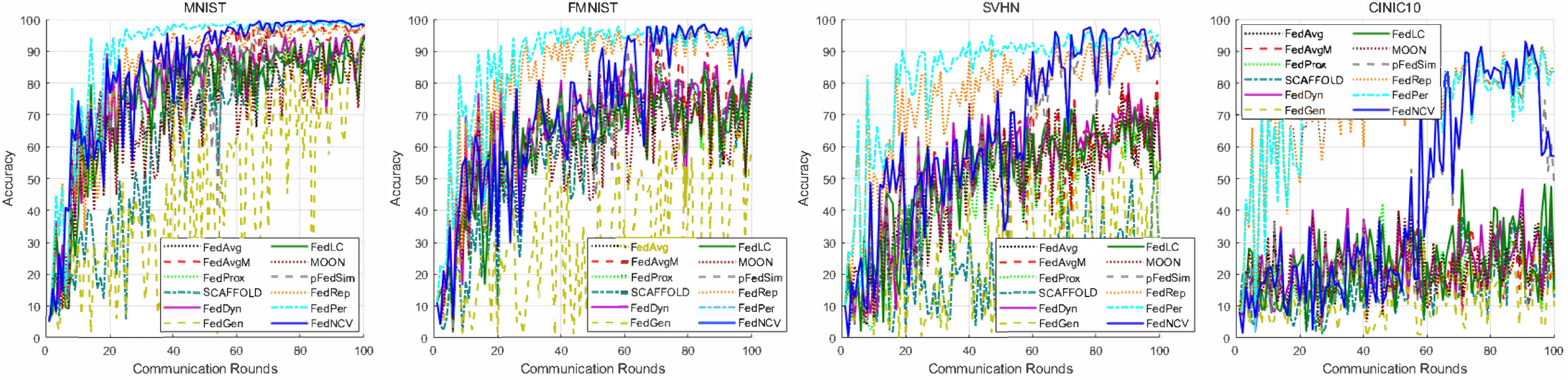}
	\end{center}
	\begin{center}
		\vspace{-1em}
		\caption{Performance evaluation of twelve solutions: pre-test accuracy vs. communication rounds across four datasets.}
		\label{Fig3}
	\end{center}
	\vspace{-3em}
\end{figure*}

\defaultbibliographystyle{unsrtnat}  
\putbib[aaai24sup]

\end{bibunit}
\end{document}